\documentclass{article}

\usepackage[preprint]{neurips_2025}

\usepackage[utf8]{inputenc} % allow utf-8 input
\usepackage[T1]{fontenc}    % use 8-bit T1 fonts
\usepackage{hyperref}       % hyperlinks
\usepackage{url}            % simple URL typesetting
\usepackage{booktabs}       % professional-quality tables
\usepackage{amsfonts}       % blackboard math symbols
\usepackage{nicefrac}       % compact symbols for 1/2, etc.
\usepackage{microtype}      % microtypography
\usepackage{xcolor}         % colors
\usepackage{amsmath}
\usepackage{graphicx} 
\usepackage{multirow}
\usepackage{xcolor}
\usepackage{wrapfig}
\usepackage{amsfonts}       % blackboard math symbols
\usepackage{amssymb}
\usepackage{mathtools}
\usepackage{amsthm}
\usepackage{multirow}
\usepackage{xspace}
\usepackage{enumitem}
\usepackage[ruled,vlined]{algorithm2e}
\usepackage[most]{tcolorbox}
\usepackage{courier}            % for a monospaced font
\usepackage{verbatim}           % for the verbatim environment
\usepackage{fancyvrb}
\usepackage{tipa}
\theoremstyle{plain}
\newtheorem{theorem}{Theorem}[section]
\newtheorem{proposition}[theorem]{Proposition}
\newtheorem{lemma}[theorem]{Lemma}

\theoremstyle{definition}

\theoremstyle{remark}

\newcommand{\proj}{\text{RAPO}\xspace}

\newcommand{\x}{x}
\newcommand{\y}{y}

\newcommand{\piref}{\pi_\text{ref}}
\newcommand{\pireftau}{\tilde{\pi}_{\text{ref}}}
\newcommand{\pitheta}{\pi_\theta}
\newcommand{\pithetaold}{\pi_{\theta_{\text{old}}}}

\newcommand{\RAPO}{\mathcal{J}_{\text{RAPO}}}
\newcommand{\J}{\mathcal{J}}
\newcommand{\JFKL}{\mathcal{J}_{\text{FKL}}}
\newcommand{\KL}{\mathbb{D}_{\text{KL}}}
\newcommand{\RKL}{ \KL \left( \pitheta || \piref \right)}
\newcommand{\FKL}{ \KL \left( \piref || \pitheta \right)}

\def\J{\mathcal{J}}

\def\controlobj1d{\mathcal{J}_{\text{actual}}}

\title{Unlocking Reasoning Capabilities in LLMs via Reinforcement Learning Exploration}

\author{
    Wenhao Deng\thanks{Equal Contribution.} \And
    Long Wei$^*$ \And
    Chenglei Yu$^*$ \And
    Tailin Wu\\
    Westlake University 
\\
\texttt{\{dengwenhao,weilong,yuchenglei,wutailin\}@westlake.edu.cn}
}

\begin{document}

\maketitle

\begin{abstract}
Reinforcement learning with verifiable rewards (RLVR) has recently enhanced the reasoning capabilities of large language models (LLMs), particularly for mathematical problem solving. However, a fundamental limitation remains: as the sampling budget increases, the advantage of RLVR-trained models over their pretrained bases often diminishes or even vanishes, revealing a strong dependence on the base model's restricted search space. We attribute this phenomenon to the widespread use of the reverse Kullback-Leibler (KL) divergence regularizer, whose mode-seeking behavior keeps the policy trapped inside the base model's support region and hampers wider exploration. To address this issue, we propose $\proj$ (Rewards-Aware Policy Optimization), an algorithm to promote broader yet focused exploration. Our method (i) utilizes the forward KL penalty to replace the reverse KL penalty for out-of-distribution exploration, and (ii) reweights the reference policy to facilitate adaptive in-distribution exploration. We train Qwen2.5-3B and 7B models with $\proj$ on the 8K SimpleRL-Zero dataset, without supervised fine-tuning, and evaluate them on AIME2024 and AIME2025. Results show that $\proj$ consistently improves problem-solving performance. Notably, $\proj$ enables models to surpass the base model's performance ceiling and solves previously intractable problems, advancing the frontier of RLVR for challenging reasoning tasks.
\end{abstract}

\section{Introduction}
Recent years have witnessed significant advancements in the reasoning capabilities of large language models (LLMs), with breakthrough systems like DeepSeek-R1 \cite{guo2025deepseek} demonstrating exceptional performance. These achievements stem not only from powerful base models but also from reinforcement learning with verifiable rewards (RLVR). By leveraging automatic verification of solution correctness as reward signals, RLVR steers model policies toward high-reward solutions, substantially enhancing reasoning capabilities.

Despite advances, RLVR approaches reveal a crucial limitation: when measured by pass@$k$ metrics, where success requires finding just one correct solution within $k$ attempts, a counterintuitive phenomenon emerges. At a low budget, RLVR-trained models consistently outperform their pre-trained base models, indicating more efficient sampling of correct answers. However, as attempts increase, this advantage not only disappears but often reverses completely: base models eventually achieve equal or superior pass@$k$ scores compared to their RL-trained versions. Recent empirical studies \cite{dang2025assessing,yue2025does} confirm this phenomenon across various model families and reasoning domains.

The finding implies that rather than endowing LLMs with fundamentally new reasoning strategies, RLVR primarily reshapes the output distribution by concentrating probability mass onto familiar reasoning paths already present in the base model’s solution space. While the redistribution increases the likelihood of high-quality responses in a small number of samples, it inadvertently narrows the model’s overall reasoning diversity. Contrary to the widespread belief that RL incentivizes continual self-improvement \cite{guo2025deepseek}, RLVR-trained models tend to become less exploratory and, even at scale, remain constrained by the inherent limitations of their base models. This contradiction prompts a crucial question:

\textbf{RQ}: \textit{How can we develop RLVR methods that enable effective exploration beyond the base model's distribution to solve previously intractable problems?}

We identify the widespread use of reverse Kullback-Leibler (KL) divergence regularization as the primary cause of the limitation. Reverse KL divergence exhibits mode-seeking behavior, which forces the fine-tuned policy to remain within high-density regions of the base model's distribution. While this stabilizes the training process, it simultaneously restricts exploration beyond the base model's support (nonzero probability) region, precluding discovery of novel solutions located beyond the support of the reference policy but with high rewards, as shown in Figure \ref{fig:overview} (a).

\begin{figure}
    \centering
    \includegraphics[width=\linewidth]{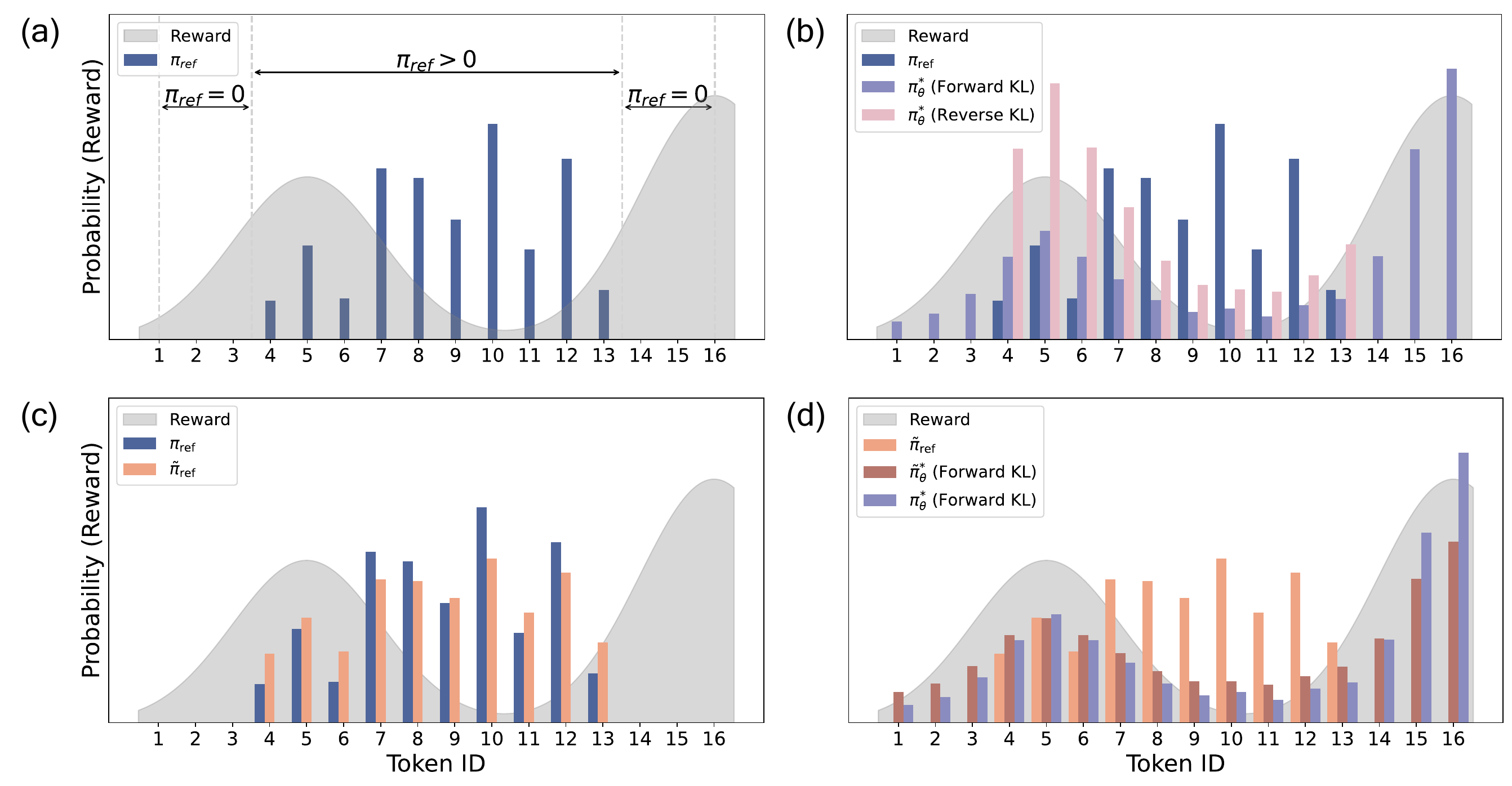}
    \caption{
    \textbf{Approach Motivation and Illustration.} The reference model $\piref$ and the reward function are shared among four subfigures. \textbf{(a)} High-reward regions with low/zero probability in the reference model are underexplored yet. \textbf{(b)} RLVR with our proposed forward KL divergence facilitates out-of-distribution exploration, overcoming reverse KL divergence limitations. \textbf{(c)} Our reward-aware reference policy reweighting mechanism for adaptive in-distribution exploration. \textbf{(d)} $\proj$, integrating the reweighted reference policy with forward KL divergence optimization, boosts exploration effectiveness.
  }
    \label{fig:overview}
\end{figure}

To overcome this issue, we introduce $\proj$ (Rewards-Aware Policy Optimization), a novel RLVR method designed to enable more effective exploration while maintaining solution quality.
 $\proj$ incorporates two key innovations:
\textbf{First}, we replace the conventional reverse KL divergence with forward KL divergence to enable \textit{out-of-distribution exploration}. Unlike reverse KL, forward KL permits the policy to assign  probability mass determined by observed rewards to regions where the reference policy has low or zero density, facilitating the discovery of solutions beyond the base model's support, as illustrated in Figure \ref{fig:overview} (b).
\textbf{Second}, we develop a reward-aware reference policy reweighting mechanism for adaptive \textit{in-distribution exploration}.
This mechanism dynamically reweights the reference policy based on observed rewards. It promotes greater exploration for low-reward regions while preserving the reference distribution in high-reward regions, as shown in Figure \ref{fig:overview} (c). 
Finally, the reweighted reference policy integrates into our forward KL divergence optimization, yielding more effective rewards-aware exploration across both out-of-distribution and in-distribution regions, as shown in Figure \ref{fig:overview} (d).

We evaluate $\proj$ by training two Qwen2.5 models (7B and 3B parameters) \cite{yang2024qwen2} on the SimpleRL-Zero dataset containing 8,000 mathematical problems. Without supervised fine-tuning, these models were tested on challenging mathematical reasoning benchmarks, including AIME2024 and AIME2025.
Experimental results demonstrate that $\proj$
significantly outperforms traditional RLVR approaches as sampling increases and can surpass the performance ceiling of base models. 
Notably, our method achieves remarkable success on problems that are entirely unsolvable by the base models under sufficient number of samples.

In summary, we make the following contributions: \textbf{(1)} We propose a RLVF exploration method that enables models to discover solutions beyond their base distribution while maintaining focused search in promising regions;
\textbf{(2)} Experiments on Qwen-2.5 models and challenging mathematical benchmarks demonstrate that our method improves reasoning capabilities across sampling budgets and solves previously intractable problems for base models.
\section{Related Work}
\subsection{RLVR for LLM Reasoning}
Recent research has demonstrated significant improvements in LLM reasoning capabilities across mathematics, programming, and scientific reasoning domains by leveraging increased computational effort during inference \cite{snell2024scaling} with pretrained base models. These approaches span a wide spectrum, from Chain-of-Thought prompting \cite{wei2022chain,yao2023tree} and process-based reward models \cite{uesato2022solving,lightman2023let,setlur2024rewarding} to Monte Carlo Tree Search \cite{feng2023alphazero,trinh2024solving} and scaled sampling with self-verification \cite{zhao2025sample,lifshitz2025multi}.
The breakthrough success of advanced models such as OpenAI-o1 \cite{OpenAIo1} and DeepSeek-R1 \cite{guo2025deepseek} has established Reinforcement Learning with Verifiable Rewards (RLVR) \cite{zelikman2022star,ouyang2022training,kumar2024training} as the dominant paradigm for enhancing LLM reasoning. RLVR optimizes rewards attached to sampled responses, shifting probability mass toward high-quality reasoning patterns and effectively transforming base models into more capable reasoning systems. This proven approach has inspired numerous follow-up studies \cite{liu2025understanding,li2025limr,yu2025dapo,muennighoff2025s1,setlur2025scaling,team2025kimi} that further refine and extend these techniques.

\subsection{Reinforcement Learning Exploration with KL Divergence}
KL divergence regularization plays a crucial role in RLVR approaches by preventing model outputs from deviating excessively from the base distribution \cite{wen2024entropy}. However, the specific formulation of KL divergence fundamentally impacts exploration behavior \cite{havrilla2024teaching,li2025choice}. As we demonstrate in the next section, reverse KL divergence inherently constrains models from exploring reasoning paths beyond the reference model's support region, establishing a performance ceiling that limits further improvements. Recent empirical studies \cite{dang2025assessing,yue2025does} have confirmed this limitation across various RLVR implementations. Forward KL divergence offers an alternative that enables policies to explore high-reward regions even with low probability in the base distribution \cite{zeng2024token}. Recent innovations like $f$-DPO \cite{wang2023beyond} and ETPO \cite{wen2024entropy} have built upon these insights to enhance exploration capabilities. Our approach combines forward KL divergence with a reward-aware reference policy reweighting mechanism to facilitate both in-distribution and out-of-distribution exploration, directly addressing the limitations of current RLVR methods.

\section{Method}\label{sec:Method}
In this section, we detail our method \proj. Section \ref{sec:preliminary} gives preliminaries and analyzes why reverse KL divergence approaches fail. In Section \ref{sec:fkl}, we introduce our forward KL divergence optimization for out-of-distribution exploration, and in Section \ref{sec:reward_fkl} we propose the reward-aware reweighting technique of the reference policy to promote in-distribution exploration. Section \ref{sec:implementation} presents the implementation (pseudocode in Algorithm \ref{alg}). All proofs are deferred to Appendix \ref{app:derivation}.

\subsection{Preliminary: Why Does Reverse KL Fail?}\label{sec:preliminary}
Let $\piref(\y|\x)$ be a pre-trained reference LLM model, which generates a response $\y$ given a question $\x$, and $\pitheta$ be the RLVR-trained model initialized by $\piref$. 
Prior RLVR methods employ the reverse Kullback-Leibler (KL) divergence, $\KL \left( \pitheta || \piref \right) = \mathbb{E}_{\pitheta}\left[\log\frac{\pitheta}{\piref}\right]$, as a regularizer to constrain policy shifts. The objective is to maximize:

\begin{equation}
\begin{split}
\J(\theta) = \mathbb{E}_{\x \sim P(x), \y\sim \pitheta(\y|\x)} 
[r(\x,\y)]  -\alpha \underbrace{\RKL}_{\text{Reverse KL divergence}}.
\label{eq:origional_RKL}
\end{split}
\end{equation}
However, reverse KL regularization inherently restricts the support of $\pitheta$ to that of $\piref$. Formally, we have the following property of $\RKL$:
\begin{lemma}\label{lemma:1}
    The optimal policy $\pitheta^{\star}$ to the problem Eq. \ref{eq:origional_RKL} satisfies
    \begin{equation}
        \pitheta^{\star}(y|x)\propto e^{\frac{r(\x,\y)}{\alpha}}\piref(\y|\x).
    \end{equation}
\end{lemma}
The proof of Lemma \ref{lemma:1} and the following Lemma \ref{lemma:2} are provided in Appendix \ref{app:derivation}. Thus, reverse KL divergence optimization can only reweight probability mass by rewards within the support of $\piref$ and never assign positive probability to regions where $\piref(y|x)=0$,
as shown in Figure \ref{fig:overview} (b). A common method to encourage exploration is to add a maximum entropy term $H(\pitheta)$, leading to the following objective:
\begin{equation}
\begin{split}
\J(\theta) = \mathbb{E}_{\x \sim P(x), \y\sim \pitheta(\y|\x)} 
[r(\x,\y)]  -\alpha \RKL + \beta H(\pitheta)
\label{eq:origional_RKL_max_entropy}
\end{split}
\end{equation}
However, this does not overcome the support limitation, as shown below.
\begin{lemma}\label{lemma:2}
    The optimal policy $\pitheta^{\star}$ to the problem Eq. \ref{eq:origional_RKL_max_entropy} satisfies
    \begin{equation}
    \pitheta^{\star}(y|x)\propto e^{\frac{r(\x,\y)}{\alpha + \beta}}\piref(\y|\x)^{\frac{\alpha}{\alpha+\beta}}.
    \end{equation}
\end{lemma}
This again shows that the reverse KL term fundamentally limits the support of $\pitheta$. 
Intuitively, minimizing $\RKL$ forces $\pitheta$ to be small wherever $\piref$ is small, since otherwise $\log \frac{\pitheta}{\piref}$ becomes large. When $\piref$ is 0 (outside the support of $\piref$), this reverse KL term forces $\pitheta$ to be also 0, regardless of whether maximum entropy is present. Targeting this limitation of reverse KL divergence, our $\proj$ enhances exploration from both out-of-distribution (outside the support of $\piref$) and in-distribution (inside the support of $\piref$) aspects.
\subsection{Forward KL Divergence: Out-of-distribution Exploration}\label{sec:fkl}
To achieve out-of-distribution exploration, we propose using the \textit{forward KL divergence} $\KL \left( \piref||\pitheta \right) = \int \piref\log\frac{\piref}{\pitheta}$ in RLVR, instead of the reverse KL divergence.
Our following analysis in this subsection aims to justify such exploration.

By introducing the forward KL divergence, the training objective with entropy maximization becomes:
\begin{equation}
\begin{split}
\JFKL(\theta) = \mathbb{E}_{\x \sim P(x), \y\sim \pitheta(\y|\x)} 
[r(\x,\y)]  -\alpha \underbrace{\textcolor{blue}{\FKL}}_{\text{Forward KL divergence}} + \beta H(\pitheta).
\label{eq:FKL_max_entropy}
\end{split}
\end{equation}
To optimize this objective under the constraint $\int_y\pitheta(y|x)\text{d}y = 1$, we introduce a Lagrange multiplier $\lambda$, leading to the following unconstrained form:
\begin{equation}
\begin{split}
\JFKL(\theta) = \mathbb{E}_{\x \sim P(x), \y\sim \pitheta(\y|\x)} 
[r(\x,\y)]  -\alpha \FKL + \beta H(\pitheta) -\lambda(\int_y\pitheta(y|x)\text{d}y - 1).
\label{eq:FKL_max_entropy_lang}
\end{split}
\end{equation}
Since LLMs generate discrete token sequences, Eq. \ref{eq:FKL_max_entropy_lang} can be rewritten in the following \textit{discrete} form:
\begin{equation}
\begin{split}
\JFKL(\theta) = &\sum_i \pitheta(y_i|x) r(\x,\y_i)  + \alpha \sum_i \piref(\y_i|x)\log \pitheta(\y_i|x) - \beta \sum_i\pitheta(\y_i|x)\log\pitheta(\y_i|x)  \\
&-\lambda(\sum_i\pitheta(\y_i|x)- 1) +\text{const}.
\label{eq:discrete_FKL_max_entropy}
\end{split}
\end{equation}
Here, $\y_i=[y^{(1)}_i,\cdots,y^{(L)}_i]$ (each $y^{(l)}_i$ is a predicted token) indexes the finite set of all output sequences up to a fixed maximum length $L$, so the summation in Eq. \ref{eq:discrete_FKL_max_entropy} is over a finite number of terms. We now formalize the optimal policy under this objective, with proof provided in Appendix \ref{app:derivation}.
\begin{proposition}\label{prop:1}
The optimal solution $(\pitheta^{\star}, \lambda^{\star})$ to the problem Eq. \ref{eq:discrete_FKL_max_entropy} satisfies:
    \begin{equation}
        \pitheta^{\star}(\y_i|\x) =
        \begin{cases}
        g(\piref(\y_i|\x), r(\x,\y_i); \alpha, \beta, \lambda), & \piref(\y_i|\x) > 0; \\
        e^{-1-\lambda/\beta+r(\x,\y_i)/\beta}, & \piref(\y_i|\x) = 0.
        \end{cases}
    \end{equation}
    where $g$ is a function determined by $\piref(\y_i|\x), r(\x,\y_i)$ and parameters $\alpha, \beta, \lambda$. The optimal multiplier $\lambda^{\star}$ is determined by the constraint $\sum_i\pitheta^{\star}(\y_i|\x)=1$.
\end{proposition}
The significance of this proposition is particularly evident in regions outside the support of $\piref$.
In these regions, sequences with higher rewards are assigned greater sampling probabilities under forward KL divergence optimization, as $ \pitheta^{\star}(\y_i|\x)\propto e^{r(\x,\y_i)/\beta}$. This contrasts sharply with reverse KL divergence optimization, which would assign zero sampling probability to such sequences outside the support of $\piref$, as compared in Figure \ref{fig:overview} (b).

\begin{wrapfigure}{R}{0.45\textwidth}
    \vspace{0cm}
    \centering
    \includegraphics[width=\linewidth]{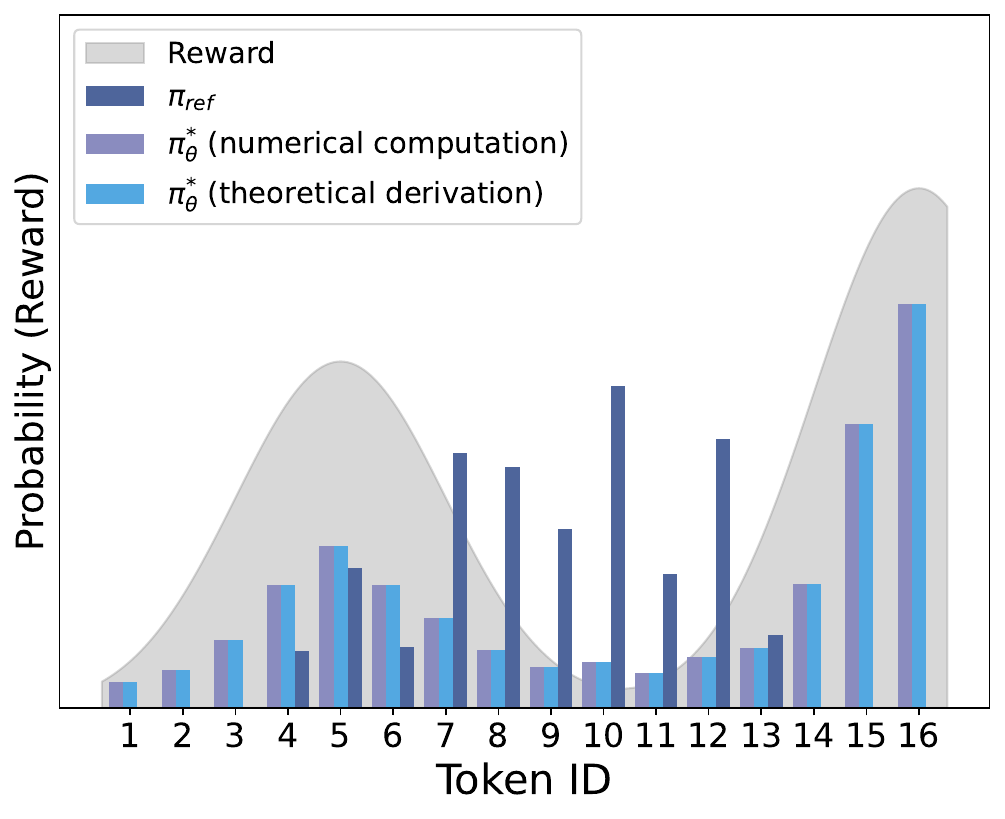}
    \caption{\textbf{Illustration of the forward KL based optimization}. The support of $\pitheta^{\star}$ extends beyond that of $\piref$ (token IDs = 0, 1, 2, 13, 14, 15). The numerical solution from gradient descent optimization of Eq. \ref{eq:discrete_FKL_max_entropy} matches the numerical root of the equation in the theoretical result of Proposition \ref{prop:1}.}
    \label{fig:2}
    \vspace{0.\baselineskip}
\end{wrapfigure}

We illustrate forward KL’s ability to assign nonzero mass outside the reference support via a toy experiment on a finite token space (Figure \ref{fig:2}). 
As token rewards change, the optimized policy adaptively boosts sampling probabilities for high-reward tokens initially absent from $\piref$.
Moreover, the policy obtained by gradient-descent computation on Eq. \ref{eq:discrete_FKL_max_entropy} matches exactly the optimal solution given by Proposition \ref{prop:1}, computed via iterative root‐finding.
\subsection{Reward-aware Reference Policy Reweighting: In-distribution Exploration}\label{sec:reward_fkl}
To complement the out-of-distribution exploration discussed above, we consider in-distribution exploration in this subsection. While entropy maximization provides a mechanism for reweighting the reference policy, it remains blind to reward signals. To adaptively balance exploration and exploitation based on reward feedback, we develop a reward-aware reference policy reweighting mechanism. Our \textit{reweighted reference policy} $\pireftau$ is formulated as:
\begin{equation}\label{eq:pireftau}
\pireftau(y|x)={\pi_\text{{ref}}^{\phi(r(x,y))}(y|x)}/{Z},
\end{equation}
where the \textit{reweight function} $\phi$ adjusts the exponent of $\piref$ according to the reward $r(\x,\y)$, and $Z=\int_y\pi_\text{{ref}}^{\phi(r(x,y))}(y|x)\text{d}y$ normalizes the distribution. In practice, $Z$ is computed by first applying $\phi$ to the discrete  probability $\piref$'s output and then summing over the vocabulary.
The function $\phi(r)$ should be monotonically increasing with values inside the range $[0,1]$.
When the reward $r$ is high, $\phi(r)$ approaches 1, reducing the degree of reweighting and keeping $\pireftau$ closer to $\piref$ to leverage existing reasoning capabilities. Conversely, when the reward $r$ is low, $\phi(r)$ approaches 0, increasing reweighting and pushing $\pireftau$ toward a more uniform distribution to encourage exploration. In particular, when $\phi(r)=1$ for any $r$, we have the special case $\pireftau=\piref$.
Choices for the design of $\phi$ is specified in Appendix \ref{app:training_details}.
Combining $\pireftau$ with the forward KL optimization in Eq. \ref{eq:FKL_max_entropy} yields our final objective:
\begin{equation}
\begin{split}
\JFKL(\theta) = \mathbb{E}_{\x \sim P(x), \y\sim \pitheta(\y|\x)} 
[r(\x,\y)]  - \alpha \KL \left( \textcolor{blue}{\pireftau} || \pitheta \right) + \beta H(\pitheta),
\label{eq:our_obj}
\end{split}
\end{equation}
whose optimal solution $\pireftau^*$ is compared with the optimal solution $\piref^*$ to Eq. \ref{eq:discrete_FKL_max_entropy} in Figure \ref{fig:overview} (d), where $\pireftau^*$ shows better diversity and obtains better consistency between sampling probabilities and rewards. Note that the second term in Eq. \ref{eq:our_obj} is generic and can be applied to a wide range of RL algorithms. Here, we incorporate it with GRPO in \cite{guo2025deepseek}, yielding our proposed $\proj$ algorithm.

\subsection{Implementation}\label{sec:implementation}
Now we present how to implement the optimization Eq. \ref{eq:our_obj} by our $\proj$ algorithm. 
\begin{equation}\label{eq:grpo}
\begin{split}
\RAPO(\theta)  = \mathbb{E}_{x \sim P(x), \{y_i\}_{i = 1}^G \sim \pithetaold(y|x)} 
\frac{1}{G}\sum_{i = 1}^G \left( g(\theta)  - \alpha \KL \left( \pireftau || \pitheta \right) + \beta H(\pitheta) \right),
\end{split}
\end{equation}
where $g(\theta)$ represents the clipped advantage-weighted policy gradient from GRPO \cite{shao2024deepseekmath}:
\begin{equation}
g(\theta)=\min \left( \frac{\pitheta(\y_i|x)}{\pithetaold(\y_i|x)} A_i, \text{clip} \left( \frac{\pitheta(\y_i|x)}{\pithetaold(\y_i|x)}, 1 - \varepsilon, 1 + \varepsilon \right) A_i \right),
\end{equation}
with normalized advantages $A_i=\frac{r_i - \text{mean}(\{r_1,\cdots,r_G\})}{\text{std}(\{r_1,\cdots,r_G\})}$ calculated from rewards $r_i=r(x,y_i)$ inside a group of $G$ solutions. 
To maximize the benefits of online exploration, during the RLVR process, we sample $\y_i$ from the $\pitheta$ instead of $\pireftau$. The forward KL divergence term $\KL \left( \pireftau||\pitheta \right)$ is calculated using the low variance estimation \cite{Schulman2020approximating}:
\begin{equation}
\KL \left(\pireftau || \pitheta\right) = \frac{\pireftau(\y_i|x)}{\pitheta(\y_i|x)}\log\frac{\pireftau(\y_i|x)}{\pitheta(\y_i|x)}  - \frac{\pireftau(\y_i|x)}{\pitheta(\y_i|x)} + 1.
\end{equation} 
This estimator has an important property: if we define a function $h(r)=r\log r-r+1$, then $\lim_{r\rightarrow{0^+}}h(r)=1$. 
Consequently, in regions where $\pireftau$ assigns low probability, $\pitheta$ can explore freely. 
This aligns with our original intention of introducing the forward KL divergence to promote exploration in regions beyond the support of the reweighted reference policy $\pireftau$. A detailed training process is presented in Algorithm \ref{alg}.

\begin{algorithm}[t]
\caption{\proj (Reward-Aware Policy Optimization)}
\label{alg}
\KwIn{reference policy $\piref$; reward function $r$; reweight function $\phi$; training dataset $\mathcal{D}$; hyperparameters $\alpha,\beta,N,M,K,G$}
\KwOut{policy $\pi_{\theta}$}
\BlankLine
Initialize $\pitheta \gets \piref$ \;
\For{$n=1$ \KwTo $N$}{
  \For{$m=1$ \KwTo $M$}{
    $\pi_{\theta_{\mathrm{old}}}\gets\pi_{\theta}$\;
    Sample batch $\mathcal{D}_b\sim\mathcal{D}$\ and  $\{\y_i\}_{i=1}^G \sim \pi_{\theta_{\mathrm{old}}}(\cdot\mid \x)$ for all $x\in\mathcal{D}_b$\;
    Compute rewards $\{r(x,y_i)\}_{i=1}^G$\;
    Compute $\hat A_i$ by group-relative advantage estimation\;
    Compute $\pireftau$\ by Eq. \ref{eq:pireftau}\;
    \For{$k=1$ \KwTo $K$}{
      Update $\pi_{\theta}$ by maximizing the objective Eq. \ref{eq:grpo}\;
    }
  }
  $\pi_{\mathrm{ref}}\gets\pi_{\theta}$\;
}
\Return{$\pi_{\theta}$}
\end{algorithm}

\section{Experiments}
In this section, we aim to answer the following question: 
\textit{Can our RAPO method transcend the reasoning limit of the base model and outperform previous RLVR approaches (e.g., GPPO) with a KL divergence regularization?
}
To answer this question, we conduct comprehensive experiments using the Qwen-2.5-7B and Qwen-2.5-3B models \cite{yang2024qwen2}, selected for their strong mathematical reasoning capabilities. 

\subsection{Experimental Setup}
Our approach contains two versions, both contains forward KL divergence regularization: 
\begin{itemize}
    \item \textbf{$\proj$-light}: $\proj$ training 
    with reweight function $\phi=1$ (no reward awareness);
    \item  \textbf{$\proj$}: $\proj$ training 
    with monotonically increasing reweight function $\phi$ as detailed in Appendix \ref{app:training_details}.
\end{itemize}
We evaluate our approach $\proj$ against the following two baselines:
\begin{itemize}
    \item Base Model: The pretrained Qwen-2.5-7B or 3B models without additional training;
    \item  GRPO-RKL:  GRPO with reverse KL divergence regularization as in \cite{guo2025deepseek}.
\end{itemize}
All experiments employ the simpleRL-reason framework \cite{zeng2025simplerl}. Following established practice \cite{chen2021evaluating,yue2025does},  we use the unbiased pass@$k$ metric
\begin{equation}
    \text{pass@}k:=\mathbb{E}_{x\sim P(x)}\left[ 1 - \frac{C_{n-c}^k}{C_n^k}\right]
\end{equation}
where $n$ ($n\geq k$) solutions are generated for each question and the number of correct solutions is denoted as $c$.

\begin{table}[tbp]
    \centering
    \caption{Comparison of mathematical reasoning performance (Pass@1024) among our $\proj$, the Base Model, and GRPO-RKL. Bold font denotes the best method, and underline denotes the second-best method.
    }
    \label{tab:comparison_baselines}
    \setlength{\tabcolsep}{4.3pt}
    \begin{tabular}{lcccccccc}
        \toprule
        \multirow{3}{*}{Method} & \multicolumn{4}{c}{Qwen-2.5 3B} & \multicolumn{4}{c}{Qwen-2.5 7B} \\
        \cmidrule(lr){2-5} \cmidrule(lr){6-9}
        & \multicolumn{2}{c}{AIME24} & \multicolumn{2}{c}{AIME25} & \multicolumn{2}{c}{AIME24} & \multicolumn{2}{c}{AIME25} \\
        & Hard & Full & Hard & Full & Hard & Full & Hard & Full \\
        \midrule
        Base model           & 0.000 & 0.646 & 0.000      &   0.600    &  0.000     &  \underline{0.777}     &    0.000   &   0.646    \\
        GRPO-RKL             & \textbf{0.125} & \underline{0.656} & 0.175      &  0.646     & 0.250      &   0.744    &   0.166    &   0.657    \\
        \midrule 
        \textbf{$\proj$-light (ours)} & \textbf{0.125} & \textbf{0.661} &  \textbf{0.300}     &   \textbf{0.706}    &  \underline{0.200}     &   0.688    &  \underline{0.220}     &  \textbf{0.800}     \\
        
        \textbf{$\proj$ (ours)} & \textbf{0.125} & 0.630 &  \underline{0.249}     &   \underline{0.653}    &  \textbf{0.350}     &   \textbf{0.809}    &  \textbf{0.479}     &  \underline{0.714}     \\
        \bottomrule
    \end{tabular}
\end{table}

\textbf{Training.} Our training dataset combines GSM8K \cite{cobbe2021training} and MATH \cite{hendrycks2021measuring}. Following the preprocessing methodology of \cite{zeng2025simplerl}, we divide the combined problems into three difficulty brackets—Easy (all GSM8K questions plus level-1 MATH items), Medium (MATH levels 1–4), and Hard (MATH levels 3–5)—with each bracket containing roughly 8,000 examples. 
The training only perform on the Hard bracket. Consistent with recent research \cite{yue2025does}, we initialize all training processes directly from the base model without any supervised fine-tuning (SFT) stage. Detailed training configurations are provided in Appendix \ref{app:training_details}.

\textbf{Evaluation.}
We assess model performance on two challenging and widely used mathematical reasoning benchmarks: 
(1) AIME24: It contains 30 questions from the American Invitational Mathematics Examination 2024; (2) AIME25: It contains 29 questions from the American Invitational Mathematics Examination 2025. 
For both datasets, we define a \textbf{Hard} subset consisting of questions that could not be solved within the maximal number \(n=2048\) of samples  by the Base Model. Evaluations are conducted on both the \textbf{Hard} subset and the \textbf{Full} dataset (comprising all questions).
During inference, we configure the model with a temperature of 0.6, top-p of 0.95, a maximum input length of 1,024 tokens, and a maximum output length of 8,196 tokens. Detailed evaluation protocols and the problem-solving prompts used for inference are provided in Appendices \ref{app:inference_details} and \ref{app:prompts}, respectively.

\subsection{Results}
Table \ref{tab:comparison_baselines} reports the mathematical reasoning performance pass@$1024$ of different methods. We choose $k=1024$ to test the limit of reasoning capabilities for each method under a sufficient number of samples ($n=2048$). From Table \ref{tab:comparison_baselines}, our $\proj$ achieves the highest inference accuracy across two models and datasets, demonstrating superior reasoning capabilities when trained using our approach. Specifically, Table \ref{tab:comparison_baselines} presents the following observations:
\begin{itemize}
    \item On the Full dataset, our method outperforms GRPO-RKL, particularly on the 7B model where we improves AIME24 accuracy from 74.4\% to 80.9\% (8.74\% relative gain) and AIME25 from 65.7\% to 80.0\% (21.77\% gain). On the 3B model, we similarly enhance AIME24 from 65.6\% to 66.1\% (0.76\% gain) and AIME25 from 64.6\% to 70.6\% (9.29\% gain), validating that our $\proj$ is more effective than GRPO-RKL, especially in larger models. 
    \item GRPO-RKL shows negligible improvement over the base model, aligning with prior findings that extensive sampling does not enhance reasoning ability \cite{yue2025does,dang2025assessing}. Our method, however, significantly outperforms the base model, achieving relative gains of 2.32\% (3B) and 4.12\% (7B) on AIME24, and striking improvements of 17.67\% (3B) and 23.84\% (7B) on AIME25. 
    \item Notably, on problems unsolvable by the base model even after $n$ attempts, our method maintains a strong edge: matching GRPO-RKL on AIME24 (3B model) while improving by 42.29\% on AIME25, and outperforming GRPO-RKL by 40\% (AIME24) and 188.55\% (AIME25) on the 7B model. These results underscore the robustness of our approach across model sizes and question difficulty levels.
\end{itemize}

\begin{figure}
    \centering
    \includegraphics[width=1\linewidth]{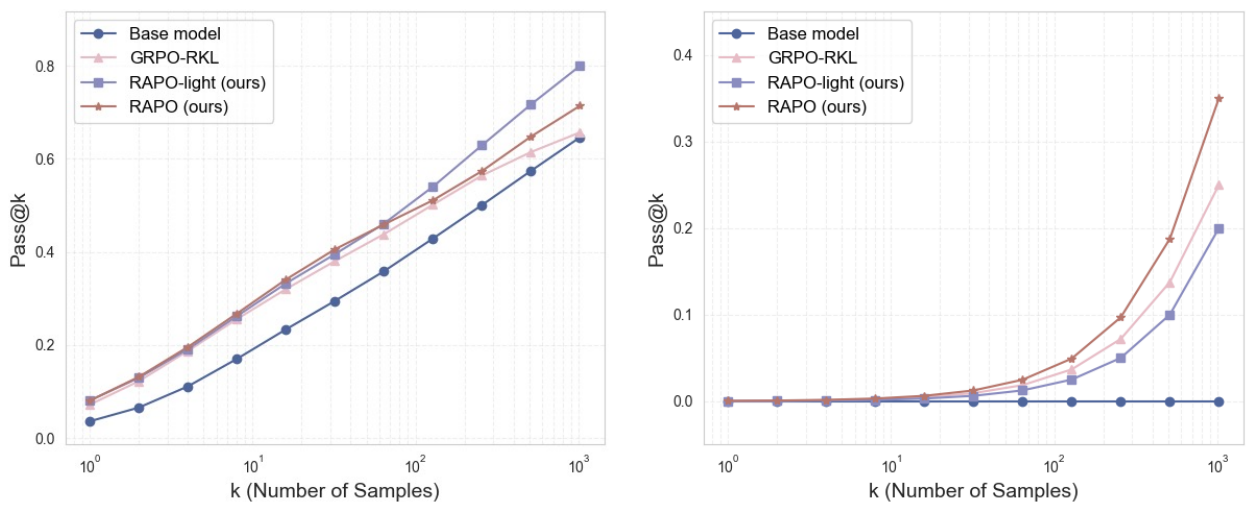}
    \caption{Comparison of mathematical reasoning performance among our $\proj$, the Base Model, and GRPO-RKL on AIME25 Full (\textit{left}) dataset and AIME24 Hard (\textit{right}) subset and Qwen2.5-7B model. Pass@$k$ is evaluated at $k=2^m$ for $m\in[0,1,\cdots,10]$. The total number of samples is $n=2048$.}
    \label{fig:fig3}
\end{figure}

Figure \ref{fig:fig3} illustrates the pass@$k$ performance of different methods as the number of sampling attempts increases on the 7B base model.
The left subfigure shows results for AIME25 Full dataset: when $k$ is small, various RLVR methods (including ours) outperform the base model, but as $k$ increases, GRPO-RKL's pass@$k$ is gradually overtaken by the base model, confirming its inability to surpass the base model's capabilities. In contrast, our RAPO method maintains a stable (RAPO) or increasing (RAPO-light) advantage over the base model, demonstrating its effectiveness in exceeding baseline performance. The right subfigure, focusing on AIME24's hard problems, reveals that while our method and GRPO-RKL start with similar pass@$k$ values, our approach exhibits a steeper upward trajectory as the number of sampling attempts increases, reaching nearly twice GRPO-RKL's performance at $k=1024$. Results at $k=1024$ align with those presented in Table \ref{tab:comparison_baselines}.

\vspace{-0.03in}

\section{Conclusion}
This work tackles a critical limitation in RLVR of LLMs, where the conventional KL-divergence constraint confines models to their base model's capabilities. 
In this work, we have introduced RAPO, a novel RLVR method designed to enable more effective exploration while maintaining solution quality. Experiments on Qwen2.5 7B and 3B models and mathematical reasoning benchmarks demonstrate that RAPO achieves consistent performance gains across sampling budgets and enables trained models to surpass base-model performance ceilings and solve previously intractable problems. 

\section{Limitation and Future Work}\label{future_work}
Our method has several limitations, offering opportunities for future research:
(1) Sample Efficiency Trade-off: While $\proj$ excels at discovering novel solutions with large sampling budgets, its advantage diminishes with limited sampling. This indicates a trade-off where broader exploration comes at the cost of efficiency in high-probability regions. Future work should develop strategies that maintain exploration capabilities while optimizing performance under small sampling budgets.
(2) Domain Applicability: Our experiments focus on mathematical reasoning tasks with clearly verifiable answers. The effectiveness of 
$\proj$ in domains with less structured rewards or in mathematical theorem proving, which requires step-by-step verification, remains unexplored. Future research should develop more fine-grained reward utilization mechanisms capable of evaluating intermediate or less structured reasoning steps.

\bibliographystyle{unsrt}
\bibliography{references}
\newpage

\clearpage
\appendix
\onecolumn

\part{Appendix} % Start the appendix part

\section{Derivation of Theoretical Analysis}
\label{app:derivation}

\begin{proof}[Proof of Lemma \ref{lemma:1}]
Consider the following variational optimization problem associated with Eq. \ref{eq:origional_RKL}:
\begin{align*}
    \mathcal{L}(\pitheta)& =\int \pitheta(y|x) \, r(x, y) \, dy\,dx - \alpha \int \pitheta(y|x) \log \frac{\pitheta(y|x)}{\piref(y|x)} dy\,dx \\
    &
    = \int \pitheta(y|x)\, r(x, y) \, dy\,dx 
    \\
    &\quad 
    - \alpha \int \pitheta(y|x) \log \pitheta(y|x)\, dy\,dx
    + \alpha \int \pitheta(y|x) \log \piref(y|x) \, dy\,dx.
\end{align*}

Here we use continuous variables in the above formula. The discrete variable case simply replaces the integral symbol with summation and replaces and continuous $y$ by $y_i$. We maximize \(L(\pitheta)\) subject to the constraint $\int \pitheta(y|x) dy\,dx = 1$ by introducing a Lagrange multiplier $\lambda$:
\begin{align*}
    \tilde{\mathcal{L}}(\pitheta, \lambda) &= 
    \int \pitheta(y|x)\, r(x, y)\, dy\,dx
    - \alpha \int \pitheta(y|x) \log \pitheta(y|x)\, dy\,dx
    \\
    &\quad + \alpha \int \pitheta(y|x) \log \piref(y|x) \, dy\,dx 
    + \lambda \left( \int \pitheta(y|x) dy\,dx - 1 \right).
\end{align*}

The stationary point of this functional is given by setting its functional derivative w.r.t $\pitheta$ to zero:
\begin{equation*}
    \frac{\delta \tilde{\mathcal{L}}}{\delta \pitheta} = r(x, y) - \alpha(1 + \log \pitheta(y|x)) + \alpha \log \piref(y|x) + \lambda = 0.
\end{equation*}
Solving for $\log \pitheta(y|x)$, we have
\begin{equation*}
    \log \pitheta(y|x) = \frac{r(x, y)}{\alpha} + \log \piref(y|x) - 1 + \frac{\lambda}{\alpha},
\end{equation*}
which implies
\begin{equation*}
    \pitheta(y|x) = e^{\frac{r(x, y)}{\alpha}} \piref(y|x) e^{(-1 + \frac{\lambda}{\alpha})}.
\end{equation*}
The factor $e^{(-1 + \frac{\lambda}{\alpha})}$ serves as a normalization constant, therefore, the optimal solution satisfies
\begin{equation*}
    \boxed{
    \pitheta^{\star} (y|x)\propto  e^{\frac{r(x, y)}{\alpha}} \piref(y|x).
    }
\end{equation*}
\end{proof}

\begin{proof}[Proof of Lemma \ref{lemma:2}]
Similar to the previous proof, the variational optimization problem associated with Eq. \ref{eq:origional_RKL_max_entropy} is
\begin{align*}
L(\pitheta) 
&= \int \pitheta(y|x) r(x, y) \, dy\,dx \\
&\quad + \alpha \int \pitheta(y|x) \log \piref(y|x) \, dy\,dx 
- (\alpha+\beta) \int \pitheta(y|x) \log \pitheta(y|x) \, dy\,dx .
\end{align*}
The Lagrangian form is:
\begin{align*}
    \tilde{\mathcal{L}}(\pitheta, \lambda) &= 
    \int \pitheta(y|x)\, r(x,y)\, dy\,dx + \alpha \int \pitheta(y|x) \log \piref(y|x)\, dy\,dx 
    \\
    &\quad - (\alpha+\beta)\int \pitheta(y|x) \log \pitheta(y|x) \, dy\,dx 
    + \lambda \left( \int \pitheta(y|x) dy\,dx - 1 \right).
\end{align*}
Compute the functional derivative w.r.t. \(\pitheta\):
\begin{align*}
\frac{\delta \tilde{L}}{\delta \pitheta} &= r(x,y) + \alpha \log \piref(y|x) - (\alpha+\beta)(1 + \log \pitheta(y|x)) + \lambda.
\end{align*}
Set this to zero for a stationary point:
\begin{equation*}
    0 = r(x,y) + \alpha \log \piref(y|x) - (\alpha+\beta)(1 + \log \pitheta(y|x)) + \lambda.
\end{equation*}
Solve $\pitheta$, we have:
\begin{equation*}
    \pitheta(y|x) = e^{\frac{r(x,y)}{\alpha+\beta}} \piref(y|x)^{\frac{\alpha}{\alpha+\beta}} e^{\frac{ \lambda - (\alpha+\beta) }{ \alpha+\beta }}.
\end{equation*}
The final exponential is a normalization constant, thus
\begin{equation*}
    \boxed{
    \pitheta^{\star}(y|x)\propto e^{\frac{r(\x,\y)}{\alpha + \beta}}\piref(\y|\x)^{\frac{\alpha}{\alpha+\beta}}.
    }
\end{equation*}
\end{proof}

\begin{proof}[Proof of Proposition \ref{prop:1}]
Taking the gradient of Eq. \ref{eq:discrete_FKL_max_entropy} w.r.t. $\pitheta(\y_i|\x)$, we have 
\begin{equation*}
    \frac{\partial \JFKL}{\partial \pitheta(\y_i|\x)}=r(\x,\y_i)+\alpha\frac{\piref(\y_i|x)}{\pitheta(\y_i|\x)} - \beta \pitheta(\y_i|\x) -\beta-\lambda.
\end{equation*}
The condition for stationary solution is
\begin{equation*}
    r(\x,\y_i)+\alpha\frac{\piref(\y_i|x)}{\pitheta(\y_i|\x)} - \beta \pitheta(\y_i|\x) -\beta-\lambda=0.
\end{equation*}
Define $F_i:[0,+\infty)\rightarrow \mathbb{R}$ as
\begin{equation*}
    F_i(u)=\alpha\frac{\piref(\y_i|x)}{u}-\beta\log u.
\end{equation*}
We have the following two cases:
\begin{itemize}
    \item When $\piref(\y_i|x)>0$, $\lim_{u\rightarrow 0^+}=+\infty$, $\lim_{u\rightarrow {+\infty}}=-\infty$. Since $F$ is continuous, there exits a solution $u^*$ for 
    \begin{equation}\label{eq:F}
        F_i(u)=\beta +\lambda-r(x,y_i).
    \end{equation}
    Define the solution for the above equation as  
    \begin{equation*}
        u^*=g(\piref(\y_i|\x), r(\x,\y_i); \alpha, \beta, \lambda).
    \end{equation*}
    \item When $\piref(\y_i|x)=0$, then the problem Eq. \ref{eq:F} reduces to 
    \begin{equation*}
        -\beta\log \pitheta(\y_i|\x)=\beta +\lambda-r(x,y_i),
    \end{equation*}
    which implies 
    \begin{equation*}
        \pitheta^{\star}(\y_i|\x) =
        e^{-1-\lambda/\beta+r(\x,\y_i)/\beta}.
    \end{equation*}
\end{itemize}
Summarizing these two cases, we get the desired optimal solutions.
\end{proof}

\section{Training Details}
\label{app:training_details}
For the rollout, the question batch size is 512 and we sample 8 solutions for each question. The coefficient of the KL penalty is 0.001. We train the 7B model for a total of 100 steps over 15 hours using eight GPUs.

\textbf{Rule-based reward}. Recent research reveals that directly using the reward model during RLVR usually suffers from the reward hacking problem \cite{amodei2016concrete}. Hence, we use the rule-based reward that assigns 1 for correct answers and 0 for incorrect ones.

\textbf{Design of $\phi(r)$}. We adopt two simple forms for $\phi(r)$:\\
\begin{itemize}
    \item Case 1 (inverse-proportional function): 
        \begin{equation}
            \phi(r)=\frac{1}{\tau_{\text{max}} - r}.
        \end{equation}
    \item  Case 2 (tanh function): 
        \begin{equation}
            \phi(r)=\frac{1+\text{tanh} (r)}{2}.
        \end{equation}
\end{itemize}
Empirically, with $\tau_{\text{max}}=2.2$  (where $\phi(r)\in[1/2.2, 1/1.2]$ for $r\in[0,1]$), the inverse-proportional function outperforms the tanh function on the Qwen2.5-3B model, while the tanh function is slightly superior on the Qwen2.5-7B model.

\section{Inference Details}
\label{app:inference_details}
For evaluation, we used a temperature of 0.6, top-p of 0.95, and a maximum generation length of 16K tokens for inference across all RLVR-trained models and the base model. We maintained consistency by using the same prompt template as in training. For the AIME 24 dataset, we sampled $n=2048$ responses per question and evaluated the unbiased Pass@1024. 

We adopted the open-source RL language model training framework Simple RL Reason \url{https://github.com/hkust-nlp/simpleRL-reason?tab=readme-ov-file} for our project. We utilized the Zero approach throughout the training process, meaning no Supervised Fine-tuning phase was involved. The training was conducted using the VeRL framework, while the inference engine is based on VLLM.

\section{The Training and Evaluation Prompt}
\label{app:prompts}
We use the following Qwen-Box prompts for RLVR training and evalution:

\DefineVerbatimEnvironment{flexverb}{Verbatim}{%
  fontfamily=courier,          % keep the monospace look
  commandchars=\\\{\},         % let \ { } act normally
  formatcom=\ttfamily          % extra safety for typewriter font
}

\begin{tcolorbox}[
  enhanced,
  colback=blue!10,
  colframe=blue!75!black,
  coltitle=white,
  title=Math Reasoning Prompt,
  fonttitle=\bfseries,
  arc=3mm,
  boxrule=1pt,
  top=5pt, bottom=5pt, left=20pt, right=20pt
]
\begin{flexverb}
<|im_start|>system
You are a helpful assistant.<|im_end|>
<|im_start|>user
\textcolor{blue}{\{question\}}
Please reason step by step,
and put your final answer within \textbackslash boxed\{\}.<|im_end|>
<|im_start|>assistant
\end{flexverb}
\end{tcolorbox}

\end{document}